\documentclass{article}
\usepackage[T1]{fontenc}
\usepackage{tgpagella}

\input{our_macros.sty}

\title{Sharp threshold for alignment of graph databases with Gaussian weights} 

\author{Luca Ganassali\footnote{INRIA, ENS, PSL Research University, Paris, France. Email: \texttt{luca.ganassali@inria.fr}.}}

\date{May 18, 2021}

\begin{document}
	
	\maketitle
	
	\begin{abstract}%
		We study the fundamental limits for reconstruction in weighted graph (or matrix) database alignment. We consider a model of two graphs where $\pi^*$ is a planted uniform permutation and all pairs of edge weights $(A_{i,j}, B_{\pi^*(i),\pi^*(j)})_{1 \leq i<j \leq n}$ are i.i.d. pairs of Gaussian variables with zero mean, unit variance and correlation parameter $\rho \in [0,1]$. We prove that there is a sharp threshold for exact recovery of $\pi^*$: if $n \rho^2 \geq (4+\eps) \log n + \omega(1)$ for some $\eps>0$, there is an estimator  $\hat{\pi}$ -- namely the MAP estimator -- based on the observation of databases $A,B$ that achieves exact reconstruction with high probability. Conversely, if $n \rho^2 \leq 4 \log n - \log \log n - \omega(1)$, then any estimator $\hat{\pi}$ verifies $\hat{\pi}=\pi$ with probability $o(1)$. 
		
		This result shows that the information-theoretic threshold for exact recovery is the same as the one obtained for detection in a recent work by \cite{Wu20}: in other words, for Gaussian weighted graph alignment, the problem of reconstruction is not more difficult than that of detection. Though the reconstruction task was already well understood for vector-shaped database alignment (that is taking signal of the form $(u_i, v_{\pi^*(i)})_{1 \leq i\leq n}$ where $(u_i, v_{\pi^*(i)})$ are i.i.d. pairs in $\dR^{d_u} \times \dR^{d_v}$), its formulation for graph (or matrix) databases brings a drastically different problem for which the hard phase is conjectured to be wide.
		
		The proofs build upon the analysis of the MAP estimator and the second moment method, together with the study of the correlation structure of energies of permutations.
	\end{abstract}
	
\maketitle


\section*{Introduction}
\paragraph*{Aligning databases}
We address the following problem: suppose that we have two databases consisting in weighted graphs represented by their adjacency matrices $A$ and $B$. For simplicity, assume that the two graphs have same size and that each individual appears in both graphs. For a given individual, its attached signal consists in weighted edges with all other users. Across databases, edges that correspond to pairs of matched individuals are correlated. We consider the following question: \emph{if the graphs are shown unlabeled (that is, if users are anonymized), is it possible to recover the corresponding matching between databases by aligning them at the sight of their correlation structure?}

Intuitively, when the matrices are correlated enough, one can learn the true matching between individuals present in the databases. In this study we investigate the precise conditions on correlation under which exact reconstruction (or perfect de-anonymization) is feasible with high probability.

\emph{De-anonymization problems} aroused great interest when \cite{Narayanan08} were able to de-anonymize an unlabeled dataset of film ratings (namely, the Netflix prize dataset) with the observation of a publicly available database (namely the Internet Movie Database), using correlations between the ratings. Since then, they have been studied in recent literature, in several versions and reformulations. The range of applications has been widened to quantifying privacy issues related to databases (\cite{Dwork08}) or social networks (\cite{Narayanan09}). 

Widespread attention was given on the \emph{graph alignment problem}, focusing on more geometrical databases (\cite{Cullina2017, Cullina18, Ding18, Fan2019Wigner, fan2019ERC, Ganassali20a}). Lots of other natural applications can be mentioned, such as \emph{pattern recognition} in image processing (\cite{Berg05, Cour07}), \emph{aligning protein interaction networks} in computational biology (\cite{Singh08}) or performing \emph{ontology alignment} in natural language processing (\cite{Haghighi05}).
	
\paragraph*{Vector-shaped and graph-shaped databases} From the theoretical point of view, fundamental limits for the deanonymisation problem are now well understood when data only consists in vectors $u,v$ of given sizes $n$ (\cite{Cullina18data, Dai19}), that is when each user has its own signal, regardless of its connections with others. In this setting, the problem can be phrased in terms of a \emph{Linear Assigment Problem (LAP)}:
\begin{equation}\label{eq:LAP}
\argmax_{\Pi} \langle \Pi u,v \rangle,
\end{equation} where the maximum runs over all permutation matrices of size $n$. Even if greedy optimization is easily seen to be exponential-time, LAP can be solved efficiently in $O(n^3)$ steps using the classical Hungarian algorithm (\cite{Kuhn55}).

Another related problem is that of linear regression with an unknown permutation, studied in \cite{Pananjady16}: this time, one observes $y = \Pi^* A x^* + w$, where $x^* \in \dR^d$ is  an unknown vector, $\Pi^*$ is an unknown $n \times n$ permutation matrix, and $w \in \dR^n$ is additive Gaussian noise. Here again, the permutation $\Pi^*$ applies only on the left side of $A$, which corresponds to row permutation.

On the other hand, when the databases are graphs, the problem is different and can be phrased this time in terms of a \emph{Quadratic Assigment Problem (QAP)}:
\begin{equation}\label{eq:QAP}
\argmax_{\Pi} \langle A ,\Pi B \Pi^T \rangle.
\end{equation} A significant difference with the previous vector-shaped setting is that this problem is known to be NP-hard in the worst case, as well as some of its approximations (\cite{Makarychev14,Pardalos94}). In the case where the signal lies in the graph structure itself -- that is, when $(A_{i,j}, B_{\pi^*(i),\pi^*(j)})_{1 \leq i<j \leq n}$ are correlated pairs of Bernoulli variables -- recent work (\cite{Cullina2017, Cullina18}) showed that there exists a sharp threshold for exact recovery, where the signal-to-noise ratio can be expressed in the correlated \ER model in terms of the size $n$ of both graphs, the marginal edge probability $p$ and the correlation parameter $s$ between edges of the two graphs. Indeed, they established that exact (resp. almost exact) reconstruction is feasible with high probability if and only if $nps \geq \log n + \omega(1)$ (resp. $nps \geq \omega(1)$). When the signal is sparser, e.g. $np = \Theta(1)$, the problem of \emph{partial graph alignment} (that is, recovering only a positive fraction of vertices) has been recently explored algorithmically (\cite{Ganassali20a}) and theoretically (\cite{Hall20}).

\paragraph*{Model of Gaussian Wigner matrices} This paper focuses on the case where signal lies in weights on edges between all pairs of nodes. In order to rigorously analyze the fundamental limits of our reconstruction problem, we will work in a probabilistic setting. The \emph{correlated Gaussian Wigner model} is first introduced by \cite{Ding18} as a standard model for random graph alignment, and has been further investigated for its own sake in recent work (\cite{Fan2019Wigner, GLM19, Wu20}).

Assume that the weighted adjacency matrices $A$ and $B$ of the two graphs $G$ and $G'$ are symmetric, and sampled as follows: first draw the planted permutation $\pi^*$ uniformly at random in $\cS_n$. Then all pairs of edge weights $(A_{i,j}, B_{\pi^*(i),\pi^*(j)})_{1 \leq i<j \leq n}$ are i.i.d. couples of normal variables with zero mean, unit variance and correlation parameter $\rho \in [0,1]$.
Since all Gaussian variables are independent from $\pi^*$, matrix $B$ can also be drawn from $A$ as follows: 
\begin{equation}\label{eq:model}
B = \rho \cdot \Pi^{*T} A\Pi^* + \sqrt{1-\rho^2}  \cdot H,
\end{equation} where $H$ is an independent copy of $A$, and $\Pi^*$ is the $n \times n$ matrix representation of permutation $\pi^*$, that is $\Pi^*_{i,j} = \ind{j = \pi^*(i)}$.

\begin{figure}[H]
	\centering
	\includegraphics[scale=1.3]{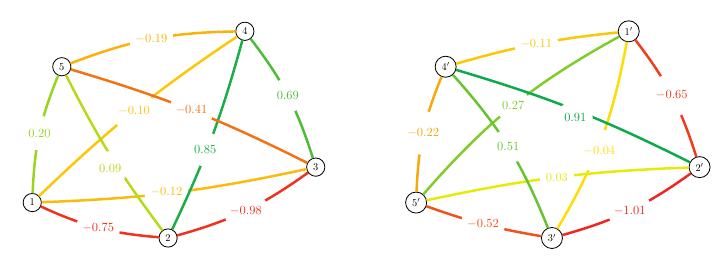}
	\caption{A sample from model \eqref{eq:model} with $n=5$. For representation, edges are colored according to their weights, and the underlying alignment is $u \mapsto u'$ for $u \in \left\lbrace 1,2,3,4,5 \right\rbrace $.}
	\label{fig:model} 
\end{figure}

\paragraph*{Detection problem} A most recent paper (\cite{Wu20}) studies fundamental limits for detection, both in correlated Gaussian weighted and correlated \ER graphs. This time, the problem is as follows: \emph{given $A,B$, are we able to distinguish between model \eqref{eq:model} and a null model, where the two graphs are just independent Gaussian weighted graphs?} Intuitively, this problem is less demanding than that of exact alignment, since the task is to detect -- wherever in the graph -- the presence of a hidden planted alignment. Under the same model \eqref{eq:model}, Y. Wu, J. Xu and S. Yu showed that detection is feasible with high probability if $n \rho^2 \geq 4 \log n$, whereas it is impossible if $n \rho^2 \leq (4-\eps) \log n$ for some $\eps>0$. Their study builds on an analysis of the likelihood ratio -- as often in detection problems. The contribution of this paper is to show that this sharp detection threshold is also that of exact reconstruction. Interestingly, for Gaussian weighted graph alignment, the problem of reconstruction is in fact not more difficult than that of detection.

After this paper was completed, the author was made aware of recent and independent work conducted by \cite{wu2021settling}, which also obtains -- among other things -- the results of this paper, albeit with different proof techniques.

\subsection*{Main results}
In the sequel, we work with the correlated Gaussian Wigner model described in \eqref{eq:model}, and establish the precise (sharp) threshold for exact recovery of $\pi^*$ in this model.
\begin{theorem}[Achievability part]
	\label{theorem_exact_r-ach}
	If for $n$ large enough
	\begin{equation}\label{eq:cond_ach_glo}
	\rho^2 \geq \frac{(4+\eps) \log n}{n}
	\end{equation}for some $\eps>0$, then there is an estimator (namely, the MAP estimator) $\hat{\pi}$ of $\pi$ given $A,B$ such that $\hat{\pi}=\pi^*$ with probability $1-o(1)$.
\end{theorem}
\begin{theorem}[Converse part]
	\label{theorem_exact_r-conv}
	Conversely, if
	\begin{equation}\label{eq:cond_imp_glo}
	\rho^2 \leq \frac{4 \log n - \log \log n - \omega(1)}{n}
	\end{equation}then any estimator $\hat{\pi}$ of $\pi$ given $A,B$ verifies $\hat{\pi}=\pi^*$ with probability $o(1)$.
\end{theorem}

\paragraph*{Computational limits of exact recovery}
For the correlated Gaussian Wigner model \eqref{eq:model}, several algorithms have been studied, usually as a first step in order to analyze further graph alignment algorithms. The state-of-the-art polynomial-time algorithms are either based on \emph{degree profiles} (\cite{Ding18}), or on a spectral method (\cite{Fan2019Wigner}). In both cases, these methods require the noise parameter $\sqrt{1-\rho^2}$ to be $O\left(\log^{-1} n\right)$. \cite{GLM19} study a simpler algorithm with lower computational complexity ($O(n^2)$ versus $O(n^3)$), requiring $\sqrt{1-\rho^2}$ to be $O(n^{-7/6})$. In any case, $\rho$ needs to tend to $1$, and the regimes in which these methods work well are far from the fundamental limits established in this paper. The present paper thus corroborates the idea that matrix alignment may be computationally hard even  in the feasibility regime. In other words, the hard phase can be conjectured to be really wide for this reconstruction problem. Proving a result of that form however remains a very thorny question.

\subsection*{Paper organization}
We first define our notations at the beginning of Section \ref{section:preliminaries}, and then establish a control on correlations between energies of permutations, using Hanson-Wright inequality. The achievability result is proved in Section \ref{section:ach}: after showing that the classical first moment method fails, we take advantage of the correlation structure established before to handle the sharp bound. Then, second moment method is applied in Section \ref{section:conv} to show that lots of small perturbations of the true underlying permutation have lower energies, establishing the converse bound. Finally, some additional proofs are deferred to Appendix \ref{appendix}. 
The proof techniques are not far from those used by \cite{Dai19}, the main novelty being the use of correlation of energies, which is essential to both achievability and impossibility result.

\section{Preliminaries}\label{section:preliminaries}
\subsection{Definitions and notations}\label{subsection:def}
For any positive integer $n$, let $[n] = \left\lbrace 1,2,\ldots,n \right\rbrace $. For two positive sequences $\left\lbrace u_n\right\rbrace $ and $\left\lbrace v_n\right\rbrace $, denote $u_n = O(v_n)$ if there exists $C>0$ such that $u_n \leq Cv_n$ for all $n$. We will also write $u_n = o(v_n)$ (resp. $u_n = \omega(v_n)$) if $u_n/v_n \to 0$ (resp. $v_n/u_n \to 0$). All limits considered are taken when $n \to \infty$.
\paragraph*{Linear algebra} We work with the canonical euclidean norm $\|\cdot \|$ on $\dR^n$, and $\langle \cdot, \cdot \rangle$ the canonical inner product on $\dR^n$ or $\dR^{n \times n}$. For any $n \times n$ matrix $M$ with real entries, its \emph{Frobenius norm} $\|M\|_{F}$ and its \emph{operator norm} $\|M\|_{\mathrm{op}}$ are defined as follows:
\begin{equation*}
\|M\|_{F}  := \left(\sum_{1\leq i,j \leq n} A_{i,j}^2\right)^{1/2} \quad \mbox{and} \quad
\|M\|_{\mathrm{op}}  := \sup_{X \in \dR^n \setminus \left\lbrace 0\right\rbrace } \frac{\| M X \|}{\|X\|}. 
\end{equation*} Note that for any normal matrix (that is, if $M^T M = M M^T$), then $\|M\|_{\mathrm{op}}$ equals $\rho(M)$, the spectral radius of $M$.

\paragraph*{Probability} When working with model \eqref{eq:model}, we will denote by $\dP_A$ (resp. $\dE_A$) the conditional probability (resp. the conditional expectation) with respect to the random matrix $A$. Throughout the paper, $\cN(\mu,v)$ denotes a Gaussian variable (resp. vector) with mean $\mu$ and variance (resp. covariance matrix) $v$. Such a Gaussian variable (resp. vector) is called \emph{standard} if $\mu=0$ and $v=1$ (resp. $v$ is the identity matrix). We say that an event $\cA_n$ happens \emph{with high probability (w.h.p)} if $\dP(\cA_n) \to 1$ when $n \to \infty$.

\paragraph*{Permutations} 
We denote by $\cS_m$ the set of permutations of $[m]$. To any permutation $\sigma \in \cS_m$, we can associate its $m \times m$ matrix representation $\Sigma$ defined by $\Sigma_{i,j} = \ind{j = \sigma(i)}$. Define $\mathcal{F}_{\sigma}$ the set of \emph{fixed points} of $\sigma$:
\begin{equation}
\mathcal{F}_{\sigma} := \left\lbrace i \in [m], \sigma(i) = i\right\rbrace,
\end{equation} and denote $f_{\sigma} := \sharp \mathcal{F}_{\sigma}$. Similarly, we define the set of \emph{unfixed points} of $\sigma$:
\begin{equation}
\mathcal{D}_{\sigma} := [m] \setminus \mathcal{F}_{\sigma} = \left\lbrace i \in [m], \sigma(i) \neq i\right\rbrace,
\end{equation} and we denote $d_{\sigma} := \sharp \mathcal{D}_{\sigma}$. For any $d \in \left\lbrace 0,\ldots,m \right\rbrace $ we define $\cS_{m,d}$ the set of permutations of $\cS_{m}$ with exactly $d$ unfixed points. Note that $\sharp \cS_{m,1}=0$ and that we have the inequality
\begin{equation}\label{eq:ineq_S_n,d}
\sharp \cS_{m,d} = \binom{m}{m-d} \sharp\left\lbrace \sigma \in \cS_{d}, F_{\sigma} = 0 \right\rbrace \leq \binom{m}{m-d} d! \leq m^d.
\end{equation} 

Similarity between two permutations $\sigma, \sigma' \in \cS_n$ is measured by their \emph{overlap}:
\begin{equation*}\label{eq:overlap}
\ov(\sigma,\sigma') := \frac{1}{n} \sum_{i=1}^{n} \mathbf{1}_{\sigma(i)=\sigma'(i)} = \frac{1}{n} f_{\sigma^{-1} \circ \sigma'}. 
\end{equation*}

Observe that on a graph of size $n$, each permutation $\sigma$ of the vertices $[n]$ has a natural extension to a canonical permutation on edges $\sigma^{\mathrm{E}} : \binom{[n]}{2} \to \binom{[n]}{2}$ defined as follows:
\begin{equation*}\label{eq:sigma_E}
\sigma^{\mathrm{E}} : e = \left\{i,j\right\} \mapsto \sigma^{\mathrm{E}}(e) = \left \{ \sigma(i),\sigma(j)\right\}.
\end{equation*} Note that the mapping $\sigma \mapsto \sigma^{\mathrm{E}}$ is one-to-one as soon as $n\geq 3$, since for all $i \in [n]$ and $j\neq j' \in [n] \setminus \left\lbrace i\right\rbrace $, edges $\sigma^{\mathrm{E}}(\left\lbrace i,j\right\rbrace )$ and $\sigma^{\mathrm{E}}(\left\lbrace i,j'\right\rbrace )$ have only one node in common, which is $\sigma(i)$. We will use the notation $\mathcal{F}^{\mathrm{E}}_{\sigma} = \mathcal{F}_{\sigma^{\mathrm{E}}}$ (resp. $\mathcal{D}^{\mathrm{E}}_{\sigma} = \mathcal{D}_{\sigma^{\mathrm{E}}}$) the set of \emph{fixed edges} (resp. \emph{unfixed edges}) of $\sigma$. Similarly we denote ${f}^{\mathrm{E}}_{\sigma} = {f}_{\sigma^{\mathrm{E}}}$ and $d^{\mathrm{E}}_{\sigma} := d_{\sigma^{\mathrm{E}}}$, for brievity.

Note that $d^{\mathrm{E}}_{\sigma}$ and are $d_{\sigma}$ are closely tied, since for all $\sigma \in \cS_n$, we have the inequality
\begin{equation}\label{eq:ineq_d}
d_{\sigma} \left(n - \frac{d_{\sigma}}{2}\right) \leq d^{\mathrm{E}}_{\sigma} \leq d_{\sigma} \left(n - \frac{d_{\sigma}-1}{2}\right).
\end{equation}
Indeed, observe that 
\begin{itemize}
\item[$(i)$] the number of fixed edges is at least the number of pairs of fixed points, and
\item[$(ii)$] the number of fixed edges is exactly the number of pairs of fixed points plus the number of pairs $(i,j), i<j$ that are exchanged by $\sigma$ (that is, the number of transpositions), this number being at most $d_{\sigma}/2$.
\end{itemize}
These remarks give that
\begin{equation*}
\binom{n-d_{\sigma}}{2} \leq \binom{n}{2} - d^{\mathrm{E}}_{\sigma} \leq \binom{n-d_{\sigma}}{2} + \frac{d_{\sigma}}{2},
\end{equation*} which directly implies \eqref{eq:ineq_d}.
\begin{rem}\label{remark:equiv_d}
Note that inequality \eqref{eq:ineq_d} gives the almost sure equivalents $d^{\mathrm{E}}_{\sigma} \sim d_{\sigma}n$ when $d_{\sigma}=o(n)$, and $d^{\mathrm{E}}_{\sigma} \sim \frac{1}{2}\alpha(2-\alpha)n^2$ when $d_{\sigma}=\alpha n$. In any case, $d^{\mathrm{E}}_{\sigma} \in \left[\frac{1}{2} d_{\sigma} n , d_{\sigma} n \right]$.
\end{rem}

\subsection{MAP estimation, relative energy of permutations}  
Since $\pi^*$ is uniformly chosen, we work in a Bayesian setting: let us evaluate the posterior probability density of $\pi^*$ given $A,B$:
\begin{flalign*}
p_{\pi^*|A,B}\left(\pi | a,b\right) & \propto p_{\pi^*,A,B}\left(\pi,a,b\right)\\
& \propto \exp\left(-  \frac{1}{2 (1-\rho^2)}\sum_{1\leq i<j\leq n} \left(B_{\pi(i),\pi(j)} - \rho A_{i,j} \right)^2 \right),
\end{flalign*} where $\propto$ indicates equality up to some factors that do not depend on $\sigma$. Define the \emph{loss function}
\begin{equation}\label{eq:L_pi}
\cL(\pi,A,B) := \sum_{1\leq i<j\leq n} \left(B_{\pi(i),\pi(j)} - \rho A_{i,j} \right)^2.
\end{equation} 

This loss function can also be viewed as the \emph{energy} associated with permutation $\pi$. Note that the posterior distribution is a Gibbs measure corresponding to this energy $\cL$, with inverse temperature $\beta =  \frac{1}{2(1-\rho^2)}$. The MAP (maximum a posteriori) estimator is thus
\begin{equation}\label{eq:MAP}
\hat{\pi}_{\MAP} := \argmax_{\pi} p_{\pi^*|A,B}\left(\pi | A,B\right) = \argmin_{\pi} \cL(\pi,A,B),
\end{equation} where the minimum is taken over all permutations $\pi \in \cS_n$. The above formulation \eqref{eq:MAP} is standard in the literature of graph and matrix alignment and meets the classical QAP formulation \eqref{eq:QAP}, since 
$$\argmin_{\pi} \cL(\pi,A,B) = \argmax_{\Pi} \langle A, \Pi B \Pi^T\rangle. $$

Theory from Bayesian optimal estimation guarantees that the best possible estimator for our exact reconstruction problem, in the Bayes risk sense, is $\hat{\pi}_{\MAP}$. Thus, if MAP estimator fails with high probability, then no estimator can succeed. This is why this estimator is often studied in exact reconstruction problems, as already done in previous works (\cite{Cullina2017,Cullina18data,Dai19}).

From now on we work conditionally on $\pi^*$ which can always be assumed to be $\mathrm{id}$ without loss of generality. More precisely, we will make the variable change $\sigma = \pi^* \circ \pi^{-1}$ ; writing $B$ as a function of $\sigma, A$ and $H$, \eqref{eq:L_pi} becomes
\begin{flalign*}
\cL(\sigma,A,H) & = \rho^2 \sum_{1\leq i<j\leq n} \left(A_{i,j} -  A_{\sigma(i),\sigma(j)} \right)^2 - 2\rho \sqrt{1-\rho^2} \sum_{1\leq i<j\leq n}  H_{i,j} \left(A_{i,j} -  A_{\sigma(i),\sigma(j)} \right)\\ & + (1-\rho^2) \sum_{1\leq i<j\leq n} H_{i,j} ^2 .
\end{flalign*} 
The loss function $\cL$ applied to the ground truth $\pi=\pi^*$ -- that is $\sigma = \id$ -- gives the energy reference $(1-\rho^2) \sum_{1\leq i<j\leq n} H_{i,j} ^2$. In order to compare any $\pi$ with $\pi^*$ -- or any $\sigma$ with $\id$ -- we further define the \emph{relative energy} of a permutation $\sigma \in \cS_n$:
\begin{flalign}
\delta(\sigma) &:= \cL(\sigma,A,H)-\cL(\mathrm{id},A,H) \nonumber\\
& = \rho^2 \sum_{1\leq i<j\leq n} \left(A_{i,j} -  A_{\sigma(i),\sigma(j)} \right)^2 - 2\rho \sqrt{1-\rho^2} \sum_{1\leq i<j\leq n}  H_{i,j} \left(A_{i,j} -  A_{\sigma(i),\sigma(j)} \right).\label{eq:delta_1} 
\end{flalign}
We next omit in our notations the dependency on $A$ and $H$ of $\delta(\sigma)$.
\begin{rem}
This relative energy $\delta$, also introduced by \cite{Cullina2017} for \ER graph alignment, is a measurement of the quality of a proposed alignment: $\delta(\sigma) \leq 0$ means that $\sigma^{-1} \circ \pi^*$ is a better alignment than $\pi^*$ for $A$ and $B$ in the posterior sense. A crucial set is then
\begin{equation*}\label{eq_Q}
\cQ := \left\lbrace \sigma \in \cS_n, \, \delta(\sigma) \leq 0 \right\rbrace.
\end{equation*}
Points of $\cQ$ are alignments on which the posterior distribution puts important weights -- at least greater weights than that of the ground truth -- or equivalently points of low energy. Note that $\id \in \cQ$. 
\end{rem}

In view of \eqref{eq:delta_1}, conditionally on $A$, $\delta(\sigma)$ is as follows:
\begin{equation}\label{eq:delta_gaussien_general}
\delta(\sigma) = \rho^2 v_\sigma - 2 \rho \sqrt{1-\rho^2} X_{\sigma},
\end{equation} where 
\begin{equation*}\label{eq:def_v_sigma}
v_\sigma := \sum_{1\leq i<j\leq n} \left(A_{i,j} -  A_{\sigma(i),\sigma(j)} \right)^2,
\end{equation*} and $X=(X_\sigma)_{\sigma \in \cS_{n}}$ is a Gaussian vector, centered, with covariance given by
\begin{equation*}\label{eq:def_c_sigma}
\cov(X_\sigma,X_{\sigma'}) = \sum_{1\leq i<j\leq n} \left(A_{i,j} -  A_{\sigma(i),\sigma(j)} \right)\left(A_{i,j} -  A_{\sigma'(i),\sigma'(j)} \right) := c_{\sigma,\sigma'}.
\end{equation*}Note that for all $\sigma \in \cS_n$, $c_{\sigma,\sigma} = v_\sigma$. Elaborating on the correlation structure of these relative energies is the object of the end of this section.

\subsection{Control of covariance structure of relative energies}
For all $\sigma, \sigma' \in \cS_{n}$, $c_{\sigma,\sigma'}$ can be written as follows
\begin{equation*}\label{eq:c_sigma_E}
c_{\sigma,\sigma'} = \sum_{e \in \binom{[n]}{2}} \left(A_e -  A_{\sigma^E(e)} \right) \left(A_e -  A_{\sigma'^E(e)} \right)
\end{equation*} and satisfies
\begin{equation*}\label{eq:c_exp}
\dE \left[c_{\sigma,\sigma'} \right] =
\sharp (\cD^{\mathrm{E}}_\sigma \cap \cD^{\mathrm{E}}_{\sigma'}) + \sharp (\cD^{\mathrm{E}}_\sigma \cap \cD^{\mathrm{E}}_{\sigma'} \cap \cF^{\mathrm{E}}_{\sigma^{-1} \circ \sigma'}).
\end{equation*} In particular,
\begin{equation*}\label{eq:v_exp}
\dE \left[v_{\sigma} \right] = d^{\mathrm{E}}_{\sigma} +d^{\mathrm{E}}_{\sigma} = 2 d^{\mathrm{E}}_{\sigma}.
\end{equation*} 

Random variables $c_{\sigma,\sigma'}$ only depend on the entries of $A$, which are Gaussian. Moreover, $c_{\sigma,\sigma'}$ being a quadratic form evaluated on a Gaussian vector, it can therefore be controlled using Hanson-Wright inequality:
\begin{lem}[Hanson-Wright inequality (\cite{HansonWright1971})] \label{lemma:HW_ineq}
	Let $X$ be a standard Gaussian vector, and $M$ a deterministic matrix. Then there exists a universal constant $c>0$ such that with probability at least $1-2\delta$:
	\begin{equation}\label{eq:HW_ineq}
	\big| X^T M X - \mathrm{Tr} M \big| \leq c \left(\|M\|_F \sqrt{\log(1/\delta)} + \|M\|_{\mathrm{op}} \log(1/\delta)\right).
	\end{equation}
\end{lem}
We refer to \cite{HansonWright1971} for a proof. Inequality \eqref{eq:HW_ineq} used in our context leads to the following

\begin{cor} \label{corollary:control_C}
	There exists a universal constant $C>0$ such that with high probability, for every $d \in \left\lbrace 2,\ldots, n \right\rbrace$, for all $\sigma,\sigma' \in \mathcal{S}_{n,d}$, 
	\begin{equation*}\label{eq:control_C}
	\big|c_{\sigma,\sigma'} - \sharp (\cD^{\mathrm{E}}_\sigma \cap \cD^{\mathrm{E}}_{\sigma'}) - \sharp (\cD^{\mathrm{E}}_\sigma \cap \cD^{\mathrm{E}}_{\sigma'} \cap \cF^{\mathrm{E}}_{\sigma^{-1} \circ \sigma'}) \big|  \leq   C d\sqrt{n \log n}.
	\end{equation*}
\end{cor}

\begin{proof}
	We first make the following observation: for any $\sigma,\sigma' \in \mathcal{S}_n$, 
	\begin{flalign*}
	c_{\sigma,\sigma'} & = \sum_{e} \left(A_{e} - A_{\sigma(e)}\right) \left(A_{e} - A_{\sigma'(e)}\right)\\
	& = A^T (I_N - \Sigma)^T(I_N - \Sigma') A,
	\end{flalign*} where $A = (A_e)_{e}$ is viewed as a standard Gaussian vector of size $N=\binom{n}{2}$, and $\Sigma$ (resp. $\Sigma'$) is the $N \times N$ permutation matrix associated with $\sigma^E$ (resp. $\sigma'^E$). Note that
	\begin{flalign*}
	\Tr ( (I_N - \Sigma)^T(I_N - \Sigma')) &= N - f^{\mathrm{E}}_{\sigma} - f^{\mathrm{E}}_{ \sigma'} +f^{\mathrm{E}}_{\sigma^{-1} \circ \sigma'} \\ 
	&\overset{(a)}{=} \sharp (\cD^{\mathrm{E}}_\sigma \cap \cD^{\mathrm{E}}_{\sigma'}) + \sharp (\cD^{\mathrm{E}}_\sigma \cap \cD^{\mathrm{E}}_{\sigma'} \cap \cF^{\mathrm{E}}_{\sigma^{-1} \circ \sigma'}),
	\end{flalign*} where $(a)$ is obtained by noticing that 
	\begin{flalign*}
	\sharp (\cD^{\mathrm{E}}_\sigma \cap \cD^{\mathrm{E}}_{\sigma'}) + \sharp (\cD^{\mathrm{E}}_\sigma \cap \cD^{\mathrm{E}}_{\sigma'} \cap \cF^{\mathrm{E}}_{\sigma^{-1} \circ \sigma'}) & = d^{\mathrm{E}}_\sigma + d^{\mathrm{E}}_{\sigma'} - \sharp (\cD^{\mathrm{E}}_\sigma \cup \cD^{\mathrm{E}}_{\sigma'}) + f^{\mathrm{E}}_{\sigma^{-1} \circ \sigma'} - \sharp(\cF^{\mathrm{E}}_{\sigma} \cup \cF^{\mathrm{E}}_{\sigma'})
	\end{flalign*} and that 
	$
	\sharp (\cD^{\mathrm{E}}_\sigma \cup \cD^{\mathrm{E}}_{\sigma'}) + \sharp(\cF^{\mathrm{E}}_{\sigma} \cup \cF^{\mathrm{E}}_{\sigma'}) = N.
	$
	For a fixed $d$ and $\sigma,\sigma' \in \mathcal{S}_{n,d}$, one has
	\begin{flalign*}
	\|(I_N - \Sigma)^T(I_N - \Sigma')\|_F &\leq \|(I_N - \Sigma')\|_F + \| \Sigma^T(I_N - \Sigma')\|_F \\
	& = 2 \|(I_N - \Sigma') \|_F \\
	& \leq 2 \sqrt{2 d^{\mathrm{E}}_{\sigma'}} \\
	& \leq 2 \sqrt{2 d n},
	\end{flalign*} where we used \eqref{eq:ineq_d} in the last step. One also has
	\begin{flalign*}
	\|(I_N - \Sigma)^T(I_N - \Sigma')\|_{\mathrm{op}} &\leq \rho(I_N-\Sigma) \times \rho(I_N - \Sigma') \\
	& \leq 2 \times 2 = 4. 
	\end{flalign*} Taking $\delta = n^{-(2d+2)}$, Lemma \ref{lemma:HW_ineq} gives that with probability at least $1-2\delta$, 
	\begin{flalign}
	\label{eq:concentration_c}
	\big|c_{\sigma,\sigma'} - \sharp (\cD^{\mathrm{E}}_\sigma \cap \cD^{\mathrm{E}}_{\sigma'}) - \sharp (\cD^{\mathrm{E}}_\sigma \cap \cD^{\mathrm{E}}_{\sigma'} \cap \cF^{\mathrm{E}}_{\sigma^{-1} \circ \sigma'}) \big|  & \leq c \left(2\sqrt{2}\sqrt{d(2d+2)} \sqrt{n \log n} + 4 (2d+2) \log n\right) \nonumber \\
	& \leq C d\sqrt{n \log n}, 
	\end{flalign}for some universal constant $C>0$. The proof is concluded by checking that this inequality holds w.h.p. for all $d$ and $\sigma,\sigma' \in \cS_{n,d}$ : the probability that at least one pair $(\sigma,\sigma')$ contradicts \eqref{eq:concentration_c} is upper bounded by
	\begin{equation*}
	\sum{d=2}^{n} \left(\sharp \cS_{n,d}\right)^2 \times 2 \delta \leq 2 \sum{d=2}^{n} 2n^{2d-2d-2} =O(1/n) = o(1).
	\end{equation*}
\end{proof}

In the rest of the paper we define the event
\begin{equation}\label{eq:event_A}
\cA := \left\lbrace \forall d \in [n], \forall \sigma,\sigma' \in \cS_{n,d}, \big|c_{\sigma,\sigma'} - \sharp (\cD^{\mathrm{E}}_\sigma \cap \cD^{\mathrm{E}}_{\sigma'}) - \sharp (\cD^{\mathrm{E}}_\sigma \cap \cD^{\mathrm{E}}_{\sigma'} \cap \cF^{\mathrm{E}}_{\sigma^{-1} \circ \sigma'}) \big|   \leq   C d\sqrt{n \log n} \right\rbrace,
\end{equation} which happens with probability $1-o(1)$ by Corollary \ref{corollary:control_C}.

\section{Achievability result}\label{section:ach}
In this section, we establish Theorem \ref{theorem_exact_r-ach}.
\subsection{Failure of first moment method}\label{subsection:first_moment}
For the achievability result, the first strategy is to use the union bound (or first moment method) to show that under condition \eqref{eq:cond_ach_glo} of Theorem \ref{theorem_exact_r-ach},
\begin{equation*}\label{eq:union_bound_ach}
\dP\left(\mbox{MAP fails} \right) = \dP\left( \hat{\pi}_{\MAP} \neq \pi \right) =o(1).
\end{equation*} As described hereafter, this naive method does not give the correct bound. Indeed, let us evaluate $\dP\left(\delta(\sigma) \leq 0 \right)$ for a given $\sigma \neq \id$. In view of the conditional distribution \eqref{eq:delta_gaussien_general} of $\delta(\sigma)$ we have
\begin{flalign*}
\dP\left(\delta(\sigma) \leq 0 \right) & = \dE \left[ \dE_A \left[\ind{\delta(\sigma) \leq 0 } \right] \right] = \dE \left[ \dP_A \left( \rho^2 v_\sigma - 2 \rho \sqrt{1-\rho^2} X_{\sigma} \leq 0\right)\right] \\
& = \dE \left[ \dP_A \left( \rho^2 v_\sigma - 2 \rho \sqrt{1-\rho^2} \sqrt{v_\sigma} \cdot \cN(0,1) \leq 0\right)\right] \\
& = \dE \left[ \dP_A \left(\cN(0,1) \geq \frac{\rho \sqrt{v_\sigma}}{2 \sqrt{1-\rho^2}}  \right)\right] \leq \dE \left[ \exp \left( - \frac{\rho^2}{8(1-\rho^2)} v_\sigma\right)\right],
\end{flalign*} where we used standard Gaussian concentration in the last inequality: $\dP\left(\cN(0,1) \geq t  \right) \leq \exp(-t^2/2)$. Note that on event $\cA$ defined in \eqref{eq:event_A} and inequality \eqref{eq:ineq_d},
\begin{equation*}\label{eq:controle_v_first_mom}
\forall d \in [n], \forall \sigma \in \cS_{n,d}, \, v_\sigma \geq 2 d^{\mathrm{E}}_{\sigma} - C d_{\sigma} \sqrt{n \log n} \geq  d^{\mathrm{E}}_{\sigma} \left(2-2\eps_n\right),
\end{equation*}
setting $\eps_n=2C \sqrt{\log n/n}$. Union bound then gives
\begin{flalign*}
\dP\left(\mbox{MAP fails} \right) & \leq \dP\left(\exists \sigma \in \cS_n \setminus \left\{\id\right\}, \; \delta(\sigma) \leq 0 \right) \\
& \leq o(1) + \sum_{\sigma \in \cS_n \setminus \left\{\id\right\}} \dE \left[ \exp \left( - \frac{\rho^2}{8(1-\rho^2)} v_\sigma\right) \ind{\cA}\right] \\
& \leq o(1) + \sum_{\sigma \in \cS_n \setminus \left\{\id\right\}} \exp \left( - \frac{\rho^2}{8(1-\rho^2)} (2-2\eps_n) d^{\mathrm{E}}_{\sigma}\right) \\
& \leq o(1) + \sum_{\sigma \in \cS_n \setminus \left\{\id\right\}} \exp \left( - \frac{\rho^2}{4} (1-\eps_n) d^{\mathrm{E}}_{\sigma}\right),
\end{flalign*}where we used $1/(1-\rho^2)>1$ in the last step. Let us now study the last sum, distinguishing the terms according to $d:=d_\sigma$: 
\begin{itemize}
	\item As long as $d=o(n)$, by Remark \ref{remark:equiv_d}, the terms behave like $\exp \left(- \frac{\rho^2}{4} (1-\eps_n) d n \right)$. By \eqref{eq:ineq_S_n,d}, $\log \sharp \cS_{n,d} \leq d \log n$ so the partial sum is small if $\frac{\rho^2}{4} (1-\eps_n) n - \log n >0$, which gives the necessary condition $\rho^2 \geq 4 \frac{\log n}{n}$. 
	
	\item However, the situation is different when it comes to large values of $d$. For instance, let us study the contribution of \emph{derangements} to the sum (that is, $\sigma$ such that $d_\sigma=n$). Note that these derangements are very numerous (their number is $\sim e^{-1} n!$). Again by Remark \ref{remark:equiv_d}, their contribution is thus of order 
	\begin{equation*}
		e^{-1} n! \exp \left(\rho^2(1-\eps_n) n^2/8(1-o(1))\right) = \exp \left(\left(n \log n - \rho^2 n^2/8\right)(1-o(1))\right),
	\end{equation*} which gives a more restrictive condition: $\rho^2 \geq 8 \frac{\log n}{n}$. 
\end{itemize}

As seen here-above, this naive first moment method enables to ensure feasibility of exact reconstruction only in the regime where $\rho^2 \geq 8 \frac{\log n}{n}$, which is not the optimal one. This bound is actually quite rough here, because the variables are substantially correlated when $d$ gets large and their contributions make the first moment explode. The next section takes advantage of these correlations in order to get access to the sharp bound.

\subsection{Improving the first moment method with correlations.}
For all $d \in \left\lbrace 2, \ldots, n \right\rbrace $, define $\cE_d$ the event:
\begin{equation*}\label{eq:def_Ed}
\cE_d := \left\lbrace \exists \sigma \in \cS_{n,d}, \; \delta(\sigma) \leq 0 \right\rbrace.
\end{equation*} In this Section we will assume that
\begin{equation*}
\rho \geq (2+\eps)\sqrt{\frac{\log n}{n}},
\end{equation*} for some $\eps>0$. Recall that we work on the event $\cA$ defined in \eqref{eq:event_A}, and that conditionally on entries of matrix $A$, we can write
\begin{equation}\label{eq:delta_gaussien}
\delta(\sigma) = \rho^2 v_\sigma - 2 \rho \sqrt{1-\rho^2} X_{\sigma},
\end{equation} where $X=(X_\sigma)_{\sigma \in \cS_{n,d}}$ is a Gaussian vector, centered, with covariance given by $\cov(X_\sigma,X_{\sigma'}) = c_{\sigma,\sigma'}$. Also note that on event $\cA$, for all $d\leq \alpha n$ and $\sigma \in \cS_{n,d}$, inequality \eqref{eq:ineq_d} gives 
\begin{flalign}\label{eq:equiv_v_alpha_n}
v_\sigma &  = 2  d^{\mathrm{E}}_{\sigma} + O(d \sqrt{n \log n}) \nonumber \\
& \geq 2  d(n- d/2) + O(d \sqrt{n \log n}) \nonumber \\
& \geq (1-o(1))2dn (1-\alpha/2) \, .
\end{flalign} In view of \eqref{eq:equiv_v_alpha_n}, as previously done in Section \ref{subsection:first_moment}, naive first moment method may suffice for $d \leq \alpha n$:
\begin{flalign*}
\dP\left(\bigcup_{2 \leq d \leq \alpha n} \cE_d \right) & \leq o(1) + \sum_{d=2}^{\alpha n} \sharp \cS_{n,d} \times \dP\left(\cN(0,1) \geq \frac{\rho \sqrt{v_{\sigma}}}{2 \sqrt{1-\rho^2}} \cap \cA\right)\\
& \leq o(1) + \sum_{d=2}^{\alpha n} \sharp \cS_{n,d} \times \dP\left(\cN(0,1) \geq (1+\eps/2) \sqrt{2d \log n (1-\alpha/2) }(1-o(1)) \right)\\
& \leq o(1) + \sum_{d=2}^{\alpha n} \exp \left(d \log n -  d \log n (1+\eps) (1-\alpha/2) + o(d \log n)\right),
\end{flalign*} which is $o(1)$ as soon as $\alpha < \alpha_0 := \frac{2 \eps}{1-\eps/2}$. It then remains to control the probabilities $\dP(\cE_d)$ for $d \geq \alpha_0 n$. 
As mentioned earlier, we take advantage of the correlation structure in \eqref{eq:delta_gaussien}. More precisely, we show that all variables $X_\sigma$ at a given level $d=\alpha n$ have substantial positive covariance when compared to their variance -- of order $\alpha (2-\alpha) n^2$ on $\cA$ by \eqref{eq:equiv_v_alpha_n} -- as shown in Figure \ref{fig:plot_f}. To do so, we derive an appropriate lower bound for $c_{\sigma,\sigma'}$ for $\sigma,\sigma' \in \cS_{n,\alpha n}$. This is the scope of the following Lemma:
\begin{lem}\label{lemma:min_corr_delta}
	With high probability, there exists a universal constant $C_1>0$ such that for any $d = \alpha n$ with fixed $\alpha>0$ and $\sigma,\sigma' \in \cS_{n,\alpha n}$:
	\begin{equation*}\label{eq:min_corr_delta}
	\cov(X_{\sigma},X_{\sigma'}) = c_{\sigma,\sigma'} \geq f(\alpha) n^2 - C_1 n^{3/2} \log^{1/2} n ,
	\end{equation*}with
	\begin{equation}\label{eq:f_alpha}
	f(\alpha) := \left\{
	\begin{array}{ll}
	\alpha^2 & \mbox{if } \alpha<1/2 \\
	\alpha^2- \frac{1}{2}(2\alpha-1)^2 & \mbox{if } \alpha \geq 1/2
	\end{array}
	\right.
	\end{equation}
	Thus for any $\eps'>0$, with high probability, for any $d=\alpha n$ with fixed $\alpha>0$,
	\begin{equation*}\label{eq:max_control}
	\max_{\sigma \in \cS_{n,\alpha n}} X_\sigma \leq \sqrt{2 \alpha \left(\alpha(2-\alpha)-f(\alpha)\right)} n^{3/2}\log^{1/2} n + (2+\eps') n \log^{1/2} n.
	\end{equation*}
\end{lem}


The proof of this Lemma is obtained by working on event $\cA$ defined in \eqref{eq:event_A}, and establishing a lower bound on $\sharp(\cD^{\mathrm{E}}_\sigma \cap \cD^{\mathrm{E}}_{\sigma'})$, which is simply the number of edges that are deranged both by $\sigma^E$ and $\sigma'^E$. It can be found in Appendix \ref{appendix:lower_bound_corr}.

\begin{figure}[H]
	\centering
	\includegraphics[width=0.8\textwidth]{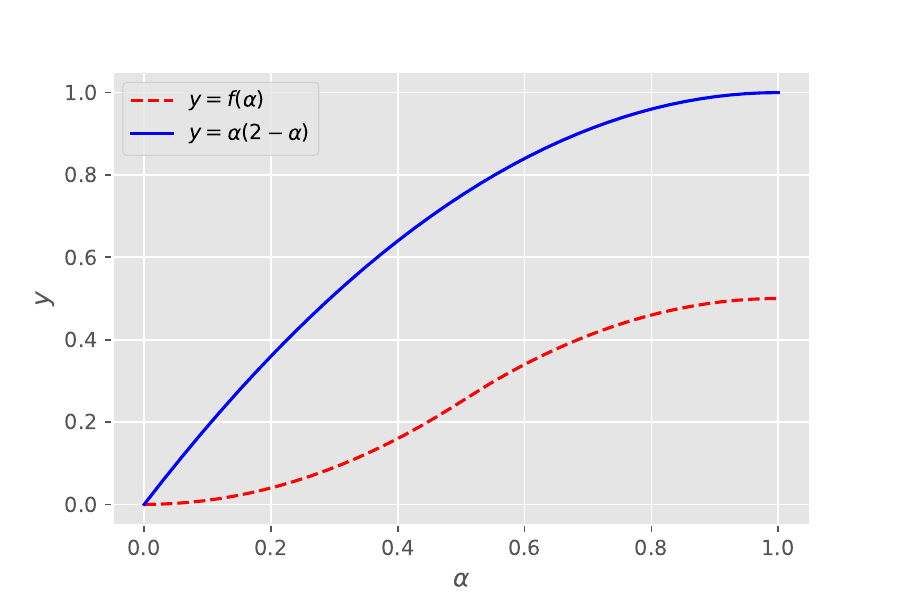}
	\caption{Plot on $[0,1]$ of normalized variance $\alpha (2-\alpha)$, together with the lower bound on the normalized covariance (function $f$) defined by \eqref{eq:f_alpha}.}
	\label{fig:plot_f} 
\end{figure}

Then, since $f(\alpha) \leq \alpha(2-\alpha)$ with elementary computations, according to Lemma \ref{lemma:min_corr_delta}, there is an event $\cB$ of probability $1-o(1)$  such that $$\max_{\sigma \in \cS_{n,d}} X_\sigma \leq (1+o(1)) \sqrt{2 \alpha \left(\alpha(2-\alpha)-f(\alpha)\right)} n^{3/2} \log^{1/2} n$$ holds for all $d=\alpha n$ with $\alpha>\alpha_0$. Note that on event $\cA \cap \cB$, for all $d  = \alpha n$ and $\sigma \in \cS_{n,d}$,
\begin{flalign*}
\rho^{-1}  \delta(\sigma) & \geq \rho v_{\sigma} - 2 \sqrt{1-\rho^2} \max_{\sigma \in \cS_{n,d}} X_\sigma \\
& \geq (1+o(1)) n^{3/2} \log^{1/2} n \left[(2+\eps) \alpha(2-\alpha) - 2 \sqrt{2 \alpha \left(\alpha(2-\alpha)-f(\alpha)\right)}\right]\\
& \geq (1+o(1)) \times 2 \times \left[\alpha(2-\alpha) - \sqrt{2 \alpha \left(\alpha(2-\alpha)-f(\alpha)\right)}\right] n^{3/2} \log^{1/2} n  \geq 0,
\end{flalign*} for $n$ large enough, since it can be easily checked (see Appendix \ref{appendix:final_function}) that 

\begin{lem}\label{lemma:final_function_study}
	For every $\alpha \in [0,1]$,
	\begin{equation} \label{eq:final_function_study}
	\alpha(2-\alpha) - \sqrt{2 \alpha \left(\alpha(2-\alpha)-f(\alpha)\right)} \geq 0.
	\end{equation}
\end{lem}

Previous computations give that 
\begin{flalign*}
\dP\left(\bigcup_{d \geq \alpha n} \cE_d\right) & \leq 1 - \dP(\cA \cap \cB) =o(1),
\end{flalign*} and ends the proof of Theorem \ref{theorem_exact_r-ach}.

\section{Converse bound: second moment method for transpositions}\label{section:conv}
In this section, we prove Theorem \ref{theorem_exact_r-conv}. As already stated in the introduction, theory from Bayesian optimal estimation guarantees that the best possible estimator for our exact reconstruction problem, in the Bayes risk sense, is $\hat{\pi}_{\MAP}$. We will show that under assumption \eqref{eq:cond_imp_glo} of Theorem \ref{theorem_exact_r-conv}, this MAP estimator fails with high probability, which implies that no estimator can succeed.

This converse bound is obtained by a second moment argument, showing that with high probability, there are  lots of permutation $\tau \neq \id$ -- in fact, transpositions -- such that $\delta(\tau)$ is negative, that is, $\tau^{-1} \circ \pi^*$ is a substantially better alignment than $\pi^*$, with lowest energy. Let us denote $\cT_n \subset \cS_n$ the set of all permutations of $[n]$ that are transpositions. 
For all $\tau \in \cT_n$, we have $d^{\mathrm{E}}_\tau = 2(n-2)$. Corollary \ref{corollary:control_C} gives that the event
\begin{equation*}\label{eq:event_D}
\cC := \left\lbrace \forall \tau, \tau' \in \cT_n, 
\big|c_{\tau,\tau'} - \sharp (\cD^{\mathrm{E}}_\tau \cap \cD^{\mathrm{E}}_{\tau'}) - \sharp (\cD^{\mathrm{E}}_\tau \cap \cD^{\mathrm{E}}_{\tau'} \cap \cF^{\mathrm{E}}_{\tau \circ \tau'}) \big|  \leq   C\sqrt{n \log n} \right\rbrace 
\end{equation*} happens with probability $1-o(1)$ for $C>0$ large enough. In particular, on $\cC$, for $C>0$ large enough,
\begin{equation*}\label{eq:control_v_tau}
\forall \tau \in \cT_n, \; \left|v_\tau - 4n\right| \leq C \sqrt{n \log n}.
\end{equation*}

In this section we are working under the assumption \eqref{eq:cond_imp_glo} that we recall here:
\begin{equation*}
\rho^2 \leq \frac{4 \log n - \log \log n - \omega(1)}{n}
\end{equation*} 
We are about to show the following: under condition \eqref{eq:cond_imp_glo}, with high probability,
\begin{equation}\label{eq_converse_transp}
\sharp \left\lbrace \tau \in \cT_n, \, \delta(\tau) < 0 \right\rbrace =\omega(1).
\end{equation} To do so, we use the classical Paley-Zygmund inequality, proven in Appendix \ref{appendix:PZ_ineq} for self-containment:
\begin{lem}[Paley-Zygmund inequality]
\label{lemma_ineq_PZ}
Let $Y$ be a real random variable with positive mean and finite variance. Then for all $0 \leq c \leq 1$,
\begin{equation*}
\dP\left(Y \geq c \,{\dE\left[Y\right]}\right) \geq \left(1 - c\right)^2 \frac{\dE\left[Y\right]^2}{\dE\left[Y^2\right]}.
\end{equation*}
Thus, in the case where $\dE\left[Y^2\right] \sim \dE\left[Y\right]^2$, taking $c\to0$ implies that $Y\geq o(\dE\left[Y\right])$ with high probability.
\end{lem}

Define
\begin{equation}\label{eq_def_X}
X := \sum_{\tau \in \cT_n} \mathbf{1}_{\delta(\tau)<0}.
\end{equation} Using a standard coupling argument in \eqref{eq_def_X}, one can see that $X$ is decreasing with $\rho$, thus we can assume without loss of generality that
\begin{equation}
\label{eq:sec_cond_conv}
\rho^2 = \frac{4 \log n - \log \log n - a_n}{n},
\end{equation} with a sequence $(a_n)_n$ such that $a_n = \omega(1)$ and $a_n = o(\log \log n)$. We compute the first moment of $X$, in view of the conditional distribution of $\delta(\tau)$ given in \eqref{eq:delta_gaussien_general}:
\begin{flalign*}
\dE\left[X\right] & \geq  \dE\left[X \mathbf{1}_\cC \right] = \frac{n(n-1)}{2} \dE\left[\dP_A \left(\cN(0,1) \geq \frac{\rho \sqrt{v_\tau}}{2 \sqrt{1-\rho^2}} \cap  \cC \right) \right]\\
& \geq \frac{n(n-1)}{2} \dE\left[(1-o(1)) \dP_A \left(\cN(0,1) \geq \frac{1}{2} \sqrt{4 \log n - \log \log n - a_n} \sqrt{4 - C n^{-1/2} \log^{1/2}n}  \right)  \right]\\
& = \frac{n(n-1)}{2} \dE\left[(1-o(1)) \dP_A \left(\cN(0,1) \geq  \sqrt{4 \log n - \log \log n - a_n} - o(1) \right)  \right]\\
& \sim \frac{n^2}{4 \sqrt{2 \pi} \sqrt{ \log n }}\exp \left(- 2\log n + \frac{\log \log n }{2} + \frac{a_n }{2}\right)\\
& = \frac{1}{4 \sqrt{2 \pi}} \exp \left( \frac{a_n}{2}  \right) \to \infty.
\end{flalign*} Note that \eqref{eq:sec_cond_conv} is thus precisely the condition ensuring that $\dE\left[X \mathbf{1}_\cC \right] \to \infty$. The second moment argument computation being a little more technical, we encapsulate it into the following Lemma:
\begin{lem}[Second moment computation of $X \mathbf{1}_\cC$]\label{lemma_second_moment}
Let $Y:= X \mathbf{1}_\cC$. Under assumption \eqref{eq:sec_cond_conv},
\begin{equation*}
\dE\left[Y^2\right]
\leq (1+o(1)) \dE\left[Y\right]^2.
\end{equation*}
\end{lem}

\begin{proof}[Proof of Lemma \ref{lemma_second_moment}]

We represent a transposition $\tau$ by its only $2-$cycle $(i \; j)$ with $i<j$. We then distinguish two cases in couples $\tau = (i \; j) \neq \tau' = (k \; \ell) \in \cT_n$:
\begin{itemize}
\item We write $\tau \cap \tau' = \varnothing$ when $\tau$ and $\tau'$ have no common point in their $2-$cycle: $i \neq k$ and $j \neq l$. When $\tau \in \cT_n$ is fixed, note that
\begin{equation*}\label{count_1}
\sharp \left\{ \tau' \in \cT_n, \,  \tau \cap \tau' = \varnothing\right\} = \frac{(n-2)(n-3)}{2}.
\end{equation*}

\item We write $\tau \cap \tau' \neq \varnothing$ when $\tau$ and $\tau'$ are different but share one common point: for instance $\tau = (3 \; 5) $ and $\tau = (5 \; 11)$ verify $\tau \cap \tau' \neq \varnothing$. When $\tau \in \cT_n$ is fixed, note that
\begin{equation*}\label{count_2}
\sharp \left\{ \tau' \in \cT_n, \,  \tau \cap \tau' \neq \varnothing\right\} = 2 (n-2).
\end{equation*}
\end{itemize} 

Note that
\begin{flalign*}
\dE\left[Y^2\right] = \dE\left[Y\right] + \sum_{\tau \in \cT_n} \sum_{\tau', \tau \cap \tau' = \varnothing} \dP(\delta(\tau)<0, \delta(\tau')<0, \cC) + \sum_{\tau \in \cT_n} \sum_{\tau', \tau \cap \tau' \neq \varnothing} \dP(\delta(\tau)<0, \delta(\tau')<0, \cC).
	\end{flalign*} We now evaluate these two sums. For this, we will need the following Lemma, which proof is deferred to Appendix \ref{appendix:corr_gaussians}.

\begin{lem}[Control of deviation probabilities for correlated Gaussians]\label{lemma:control_dev_cor_gaussian}
Let $Z_1, Z_2$ be two Gaussian variables with mean $0$, variance $1$ and correlation $\alpha_n \in [0,1]$. For any $t_n$ such that $t_n \to \infty$, 
\begin{itemize}
	\item[$(i)$] If $\alpha_n t_n \to 0$, then for $n$ large enough
	\begin{equation}\label{eq:control_dev_gaussian_1}
	\dP \left(Z_1 > t_n, Z_2 > t_n \right) \leq e^{-2t_n^2} + (1+o(1))\dP \left(Z_1 > t_n\right) \dP \left(Z_2 > t_n\right).
	\end{equation} 
	\item[$(ii)$] More generally,
	\begin{equation}\label{eq:control_dev_gaussian_2}
	\dP \left(Z_1 > t_n, Z_2 > t_n \right) \leq (1+o(1))\frac{1+\alpha_n}{\sqrt{2\pi} \, t_n}\exp\left(- \frac{t_n^2}{1+ \alpha_n} \right).
	\end{equation} 
\end{itemize}
\end{lem}

\paragraph*{First case: $\tau \cap \tau' = \varnothing$.} Without loss of generality we can assume that $\tau = (1 \; 2)$ and $\tau' = (3 \; 4)$. The following diagram shows the simple action of $\tau$ and $\tau'$ on an interesting (overlapping) subset of edges.
\begin{center}
	\begin{tabular}{ c c c }
		$\left\lbrace 1,3\right\rbrace $ & $\overset{\tau}{\longleftrightarrow} $& $\left\lbrace 2,3\right\rbrace $\\ 
		{\scriptsize $\tau'$} $\updownarrow$&  & $\updownarrow$ {\scriptsize $\tau'$}  \\  
		$\left\lbrace 1,4\right\rbrace $ & $\overset{\tau}{\longleftrightarrow}$ & $\left\lbrace 2,4\right\rbrace $   
	\end{tabular}
\end{center} 

We then see that $\sharp (\cD^{\mathrm{E}}_\tau \cap \cD^{\mathrm{E}}_{\tau'}) + \sharp (\cD^{\mathrm{E}}_\tau \cap \cD^{\mathrm{E}}_{\tau'} \cap \cF^{\mathrm{E}}_{\tau \circ \tau'}) = 4 + 0 = 4$. So, denoting $\alpha_{\tau,\tau'} := \frac{c_{\tau,\tau'}}{\sqrt{v_\tau v_{\tau'}}}$, on $\cC$,
\begin{equation*}\label{eq:case_1_alpha}
\left|\alpha_{\tau,\tau'}\right| \leq \frac{C\sqrt{n \log n}+4}{4n-C\sqrt{n \log n}} = O \left(\sqrt{\frac{\log n}{n}}\right).
\end{equation*}
In view of the conditional distribution of $\delta(\tau)$ given in \eqref{eq:delta_gaussien_general}:

\begin{equation}\label{eq:sum_case1}
\sum_{\tau \in \cT_n} \sum_{\tau', \tau \cap \tau' = \varnothing} \dP(\delta(\tau)<0, \delta(\tau')<0, \cC) = (1-o(1)) \sum_{\tau \in \cT_n} \sum_{\tau', \tau \cap \tau' = \varnothing} \dP \left( Z_\tau > t_n, \, Z_{\tau'} >t_n\right),
\end{equation} with $t_n = \sqrt{4\log n - \log \log n - a_n}$, where $Z_{\tau}, Z_{\tau'}$ are two Gaussian variables of mean $0$, with correlation coefficient $\alpha_n$ of order $O(\log^{1/2}n^{-1/2})$. Since $\alpha_n t_n \to 1$, by lemma \ref{lemma:control_dev_cor_gaussian} case $(i)$, the sum in \eqref{eq:sum_case1} is upper bounded by
\begin{flalign*}
 & (1-o(1))\frac{n(n-1)}{2} \times \frac{(n-2)(n-3)}{2} \times \left[C e^{-2t_n^2} + (1-o(1))\dP \left(Z_1 > t_n\right) \dP \left(Z_2 > t_n\right)\right] \\
 & \leq (1+o(1))\dE\left[Y\right]^2.
\end{flalign*}

\paragraph*{Second case: $\tau \cap \tau' \neq \varnothing$.} Without loss of generality we can assume that $\tau = (1 \; 2)$ and $\tau' = (2 \; 3)$. We can immediately deduce that $\sharp (\cD^{\mathrm{E}}_\tau \cap \cD^{\mathrm{E}}_{\tau'}) + \sharp (\cD^{\mathrm{E}}_\tau \cap \cD^{\mathrm{E}}_{\tau'} \cap \cF^{\mathrm{E}}_{\tau \circ \tau'}) = (n-2) + 0 = n-2$. So, denoting $\alpha_{\tau,\tau'} := \frac{c_{\tau,\tau'}}{\sqrt{v_\tau v_{\tau'}}}$, on $\cC$,
\begin{equation*}\label{eq:case_2_alpha}
\left|\alpha_{\tau,\tau'}\right| \leq \frac{C\sqrt{n \log n}+n-2}{4n-C\sqrt{n \log n}} \sim \frac{1}{4}.
\end{equation*}
Again, in view of the conditional distribution of $\delta(\tau)$ given in \eqref{eq:delta_gaussien_general}:

\begin{equation}\label{eq:sum_case2}
\sum_{\tau \in \cT_n} \sum_{\tau', \tau \cap \tau' \neq \varnothing} \dP(\delta(\tau)<0, \delta(\tau')<0, \cC) = (1-o(1)) \sum_{\tau \in \cT_n} \sum_{\tau', \tau \cap \tau' \neq \varnothing} \dP \left( Z_\tau > t_n, \, Z_{\tau'} >t_n\right),
\end{equation} with $t_n = \sqrt{4\log n - \log \log n - a_n}$, where $Z_{\tau}, Z_{\tau'}$ are two Gaussian variables of mean $0$, with correlation coefficient $\alpha_n\sim 1/4$. By Lemma \ref{lemma:control_dev_cor_gaussian} case $(ii)$, the sum in \eqref{eq:sum_case2} is upper bounded by
\begin{flalign*}
(1-o(1))\frac{n(n-1)}{2} \times 2(n-2) \times \left[ (1+o(1))\frac{1+\alpha_n}{\sqrt{2\pi} \, t_n}\exp\left(- \frac{t_n^2}{1+ \alpha_n} \right)\right]\\
\leq C'' n^3 \log^{-1/2}(n)  \exp\left(- \frac{16}{5} \log n + o(\log n) \right)=o(1)=o(\dE\left[Y\right]^2).
\end{flalign*}
\end{proof}	
	
Lemma \ref{lemma_second_moment} together with Payley-Zigmund inequality (Lemma \ref{lemma_ineq_PZ}) implies that $Y \geq o\left(\mathbb{E}[Y]\right)$ with high probability and thus proves \eqref{eq_converse_transp} and the converse result of Theorem \ref{theorem_exact_r-conv}.

\begin{rem}
We have shown here that under condition \eqref{eq:cond_imp_glo}, there is with high probability a great number of negative relative energy points near the ground truth, none of them being of significant interest to recover \emph{exactly} our permutation. 
We may also study this relative energy far from the planted permutation, which would be interesting to address the problem of almost exact (resp. partial) alignment, which consists in finding an estimator $\hat{\pi}$ that coincides with $\pi$ on at least $n - o(n)$ (resp. some positive fraction of $n$) points. In the light of our result which shows that exact recovery is not more difficult than detection, we can also conjecture that the same threshold $n \rho^2 / \log n = 4$ is sharp for the tasks of almost exact and partial recovery.
\end{rem}

\newpage
\section*{Acknowledgments}
The author would like to thank Léo Miolane, Marc Lelarge, and Laurent Massoulié for helpful discussions.
This work was supported by the French government under management of Agence Nationale de la Recherche as part of the “Investissements d’avenir” program, reference ANR19-P3IA-0001 (PRAIRIE 3IA Institute).

\bibliographystyle{alpha}
\bibliography{biblio}

\begin{thebibliography}{FMWX19b}

\bibitem[BBM05]{Berg05}
A.~C. {Berg}, T.~L. {Berg}, and J.~{Malik}.
\newblock Shape matching and object recognition using low distortion
  correspondences.
\newblock In {\em 2005 IEEE Computer Society Conference on Computer Vision and
  Pattern Recognition (CVPR'05)}, volume~1, pages 26--33 vol. 1, 2005.

\bibitem[CK17]{Cullina2017}
Daniel Cullina and Negar Kiyavash.
\newblock Exact alignment recovery for correlated {E}rd{\H{o}}s-{R}{\'{e}}nyi
  graphs, 2017.

\bibitem[CKMP18]{Cullina18}
Daniel Cullina, Negar Kiyavash, Prateek Mittal, and H.~Vincent Poor.
\newblock Partial recovery of {E}rd{\H{o}}s-{R}{\'{e}}nyi graph alignment via
  k-core alignment.
\newblock {\em CoRR}, abs/1809.03553, 2018.

\bibitem[CMK18]{Cullina18data}
Daniel Cullina, P.~Mittal, and N.~Kiyavash.
\newblock Fundamental limits of database alignment.
\newblock {\em 2018 IEEE International Symposium on Information Theory (ISIT)},
  pages 651--655, 2018.

\bibitem[CSS07]{Cour07}
Timothee Cour, Praveen Srinivasan, and Jianbo Shi.
\newblock Balanced graph matching.
\newblock In B.~Sch\"{o}lkopf, J.~C. Platt, and T.~Hoffman, editors, {\em
  Advances in Neural Information Processing Systems 19}, pages 313--320. MIT
  Press, 2007.

\bibitem[DMWX18]{Ding18}
Jian {Ding}, Zongming {Ma}, Yihong {Wu}, and Jiaming {Xu}.
\newblock {Efficient random graph matching via degree profiles}.
\newblock {\em arXiv e-prints}, page arXiv:1811.07821, Nov 2018.

\bibitem[Dwo08]{Dwork08}
Cynthia Dwork.
\newblock Differential privacy: A survey of results.
\newblock In Manindra Agrawal, Dingzhu Du, Zhenhua Duan, and Angsheng Li,
  editors, {\em Theory and Applications of Models of Computation}, pages 1--19,
  Berlin, Heidelberg, 2008. Springer Berlin Heidelberg.

\bibitem[ECK19]{Dai19}
Osman {Emre Dai}, Daniel {Cullina}, and Negar {Kiyavash}.
\newblock {Database Alignment with Gaussian Features}.
\newblock {\em arXiv e-prints}, page arXiv:1903.01422, March 2019.

\bibitem[FMWX19a]{Fan2019Wigner}
Zhou Fan, Cheng Mao, Yihong Wu, and Jiaming Xu.
\newblock Spectral graph matching and regularized quadratic relaxations {I}:
  The gaussian model, 2019.

\bibitem[FMWX19b]{fan2019ERC}
Zhou Fan, Cheng Mao, Yihong Wu, and Jiaming Xu.
\newblock Spectral graph matching and regularized quadratic relaxations {II}:
  {E}rd{\H{o}}s-{R}{\'{e}}nyi graphs and universality, 2019.

\bibitem[GLM19]{GLM19}
L.~{Ganassali}, M.~{Lelarge}, and L.~{Massouli{\'e}}.
\newblock {Spectral alignment of correlated Gaussian random matrices}.
\newblock {\em arXiv e-prints}, page arXiv:1912.00231, November 2019.

\bibitem[GM20]{Ganassali20a}
Luca Ganassali and Laurent Massouli\'e.
\newblock From tree matching to sparse graph alignment.
\newblock volume 125 of {\em Proceedings of Machine Learning Research}, pages
  1633--1665. PMLR, 09--12 Jul 2020.

\bibitem[HM20]{Hall20}
Georgina {Hall} and Laurent {Massouli{\'e}}.
\newblock {Partial Recovery in the Graph Alignment Problem}.
\newblock {\em arXiv e-prints}, page arXiv:2007.00533, July 2020.

\bibitem[HNM05]{Haghighi05}
Aria~D. Haghighi, Andrew~Y. Ng, and Christopher~D. Manning.
\newblock Robust textual inference via graph matching.
\newblock In {\em Proceedings of the Conference on Human Language Technology
  and Empirical Methods in Natural Language Processing}, HLT '05, pages
  387--394, Stroudsburg, PA, USA, 2005. Association for Computational
  Linguistics.

\bibitem[HW71]{HansonWright1971}
D.~L. Hanson and F.~T. Wright.
\newblock A bound on tail probabilities for quadratic forms in independent
  random variables.
\newblock {\em Ann. Math. Statist.}, 42(3):1079--1083, 06 1971.

\bibitem[Kuh55]{Kuhn55}
H.~W. Kuhn.
\newblock The hungarian method for the assignment problem.
\newblock {\em Naval Research Logistics Quarterly}, 2(1‐2):83--97, 1955.

\bibitem[MMS14]{Makarychev14}
Konstantin Makarychev, Rajsekar Manokaran, and Maxim Sviridenko.
\newblock Maximum quadratic assignment problem: Reduction from maximum label
  cover and lp-based approximation algorithm.
\newblock {\em CoRR}, abs/1403.7721, 2014.

\bibitem[NS08]{Narayanan08}
A.~{Narayanan} and V.~{Shmatikov}.
\newblock Robust de-anonymization of large sparse datasets.
\newblock In {\em 2008 IEEE Symposium on Security and Privacy (sp 2008)}, pages
  111--125, May 2008.

\bibitem[NS09]{Narayanan09}
A.~{Narayanan} and V.~{Shmatikov}.
\newblock De-anonymizing social networks.
\newblock In {\em 2009 30th IEEE Symposium on Security and Privacy}, pages
  173--187, May 2009.

\bibitem[PRW94]{Pardalos94}
Panos Pardalos, Franz Rendl, and Henry Wolkowicz.
\newblock {\em The Quadratic Assignment Problem: A Survey and Recent
  Developments}, pages 1--42.
\newblock 08 1994.

\bibitem[PWC16]{Pananjady16}
Ashwin {Pananjady}, Martin~J. {Wainwright}, and Thomas~A. {Courtade}.
\newblock {Linear Regression with an Unknown Permutation: Statistical and
  Computational Limits}.
\newblock {\em arXiv e-prints}, page arXiv:1608.02902, August 2016.

\bibitem[SXB08]{Singh08}
Rohit Singh, Jinbo Xu, and Bonnie Berger.
\newblock Global alignment of multiple protein interaction networks with
  application to functional orthology detection.
\newblock {\em Proceedings of the National Academy of Sciences},
  105(35):12763--12768, 2008.

\bibitem[WXY20]{Wu20}
Yihong {Wu}, Jiaming {Xu}, and Sophie~H. {Yu}.
\newblock {Testing correlation of unlabeled random graphs}.
\newblock {\em arXiv e-prints}, page arXiv:2008.10097, August 2020.

\bibitem[WXY21]{wu2021settling}
Yihong Wu, Jiaming Xu, and Sophie~H. Yu.
\newblock Settling the sharp reconstruction thresholds of random graph
  matching, 2021.

\end{thebibliography}

\newpage
\appendix

\section{Additional proofs}\label{appendix}

\subsection{Proof of Lemma \ref{lemma:min_corr_delta}: lower bound on correlations of relative energies}\label{appendix:lower_bound_corr}
\begin{proof}
	Recall that we work under event $\cA$. Fix $\alpha \in (0,1]$ and take $d=\alpha n$ and $\sigma,\sigma' \in \cS_{n,d}$. The proof is obtained by establishing a fine lower bound on $\sharp(\cD^{\mathrm{E}}_\sigma \cap \cD^{\mathrm{E}}_{\sigma'})$, which is simply the number of edges that are deranged both by $\sigma^E$ and $\sigma'^E$. In order to establish this lower bound, let us assume that $\sigma$ and $\sigma'$ have $\sharp(\cD_\sigma \cap \cD_{\sigma'}) = \beta n$ common unfixed points, with $\beta \in [0,\alpha]$. We then form edges in $\cD^{\mathrm{E}}_\sigma \cap \cD^{\mathrm{E}}_{\sigma'}$ in the following way:
	\begin{itemize}
		\item First, by taking all pairs but the pairs made of points in the complement of $\cD_\sigma \cap \cD_{\sigma'}$ and those made of pairs $(i,j)$ that are transpositions of $\sigma$ or $\sigma'$, we obtain at least $\frac{1}{2}\beta(2-\beta) n^2 - \alpha n$ edges.
		\item Then, add new edges made of one extremity in $\cD_{\sigma} \setminus \cD_{\sigma'}$ and one in $\cD_{\sigma'} \setminus \cD_{\sigma}$. Since $\cD_{\sigma}$ (resp $\cD_{\sigma}$) is stable by $\sigma$ (resp. by $\sigma'$), all these $(\alpha-\beta)^2 n^2$ edges are in $\cD^{\mathrm{E}}_\sigma \cap \cD^{\mathrm{E}}_{\sigma'}$.
	\end{itemize} Finally we formed $g(\alpha,\beta)n^2 - \alpha n$ edges, with 
	\begin{equation}
	g(\alpha,\beta) := \frac{1}{2} \beta^2 + (1-2\alpha) \beta + \alpha^2,
	\end{equation} which is minimal on $[0,\alpha]$ at $\beta = 2\alpha-1$ if $\alpha \geq 1/2$, or at $\beta=0$ if $\alpha<1/2$. In any case, this minimum is $f(\alpha)$. The first inequality is established by applying inequality \eqref{eq:event_A} of event $\cA$.\\
	
	For the second part, consider a centered vector $Z= (Z_{\sigma})_{\sigma \in \cS_{n,\alpha n}}$ such that all $Z_{\sigma}$ have same variance $v_\alpha$ and $\cov(Z_\sigma,Z_{\sigma'}) = c_\alpha $ for $\sigma \neq \sigma'$, with $v_\alpha, c_\alpha $ defined as follows:
	\begin{flalign*}
	v_\alpha &:= \alpha (2-\alpha)n^2 - C_1 n^{3/2}\log^{1/2}n,\\
	c_\alpha &:=  f(\alpha)n^2 - C_1 n^{3/2}\log^{1/2}n.
	\end{flalign*}for some $C_1>0$ large enough. Note that on event $\cA$, for all $\alpha \in (0,1]$, all $\sigma,\sigma' \in \cS_{n,\alpha n}$, $$\cov(Z_\sigma,Z_{\sigma'}) \leq \cov(X_\sigma,X_{\sigma'}),$$ so one has that for all $t>0$,
	\begin{equation}\label{eq:comparison_max_cov}
	\dP\left(\max_{\sigma \in \cS_{n,\alpha n}} X_\sigma > t \; \cap \cA\right) \leq \dP\left(\max_{\sigma \in \cS_{n,\alpha n}} Z_\sigma > t \right).
	\end{equation} We now control the right-hand side of \eqref{eq:comparison_max_cov} with this classical Lemma, which proof is find hereafter in Appendix \ref{appendix:max_corr_Gaussians}:
	\begin{lem}[Maximum of totally correlated Gaussian variables]\label{lemma:max_TC_gaussian}
		Let $Z$ be a centered Gaussian vector of size $N$, such that all $Z_i$ have same variance $v$ and $\cov(Z_i,Z_j) = c $ for $i \neq j$. Then
		\begin{equation}\label{eq:max_TC_gaussian}
		\dP\left(\max_{1 \leq i \leq N} Z_i > \sqrt{2(v-c) \log N} + 2\sqrt{v\log \log N} \right) \leq \frac{2}{\log N}.
		\end{equation}
	\end{lem}
	Note that for $v_\alpha, c_\alpha$ previously defined, one has
	\begin{equation}\label{eq:ineq_Z_1}
	\sqrt{2(v_\alpha-c_\alpha)\log \sharp \cS_{n,\alpha n}} \leq \sqrt{2 \alpha (\alpha(2-\alpha)-f(\alpha))} n^{3/2}\log^{1/2} n,
	\end{equation} and for $n$ large enough,
	\begin{equation}\label{eq:ineq_Z_2}
	2\sqrt{v_\alpha \log\log \sharp \cS_{n,\alpha n}} \leq 2 \sqrt{ \alpha(2-\alpha)} n \sqrt{\log n + \log \log n} \leq (2+\eps') n \log^{1/2} n.
	\end{equation} Finally, we use equations \eqref{eq:comparison_max_cov}--\eqref{eq:ineq_Z_2} to conclude that for $n$ large enough:
	\begin{flalign*}
	&\dP\left(\exists d = \alpha n , \alpha>\alpha_0, \; \max_{\sigma \in \cS_{n,d}} X_\sigma > \sqrt{2 \alpha \left(\alpha(2-\alpha)-f(\alpha)\right)} n^{3/2}\log^{1/2} n + (2+\eps') n \log^{1/2} n\right)\\
	&\leq 1-\dP(\cA) + \sum_{d=\alpha n, \, \alpha>\alpha_0} \dP\left(\max_{\sigma \in \cS_{n,\alpha n}} Z_i > \sqrt{2(v_\alpha-c_\alpha)\log \sharp \cS_{n,\alpha n}} + 2\sqrt{v_\alpha \log\log \sharp \cS_{n,\alpha n}} \right)\\
	&\leq o(1) + \sum_{d=\alpha n, \, \alpha>\alpha_0} \frac{2}{\log \sharp \cS_{n,\alpha n}} \leq o(1) + \frac{2n}{\log \sharp \cS_{n,\alpha_0 n}} =  o(1) + \frac{2}{\alpha_0 \log n} = o(1),
	\end{flalign*} and Lemma \ref{lemma:min_corr_delta} is proved.
\end{proof}

\subsection{Proof of Lemma \ref{lemma:max_TC_gaussian}: maximum of totally correlated Gaussian variables}\label{appendix:max_corr_Gaussians}

\begin{proof}
	Let us make a change of variables which preserves the joint distribution:
	\begin{equation*}
	\left(Z_1, Z_2, \ldots, Z_N\right) = \left(\sqrt{c} \,\xi_0+\sqrt{v-c}\, \xi_1, \ldots,\sqrt{c}\, \xi_0+\sqrt{v-c}\, \xi_N\right),
	\end{equation*} where $\xi_0, \ldots, \xi_N$ are independent standard Gaussian random variables. The maximum thus writes
	\begin{equation*}
	\max_{1 \leq i \leq N} Z_i = \sqrt{c} \, \xi_0 + \sqrt{v-c} \max_{1 \leq i \leq N} \xi_i
	\end{equation*}
	Then, with the classical inequality $\dP \left(\cN(0,1) \geq t\right) \leq e^{-t^2/2}$, then with probability at least $1-1/(\log N)$, one has:
	\begin{equation*}\label{eq:control_gaussian}
	\sqrt{c} \, \xi_0 \leq \sqrt{2c \log \log N}, \quad \mbox{and} \quad \sqrt{v-c} \max_{1 \leq i \leq N} \xi_i \leq \sqrt{2(v-c) \log N \left(1+\frac{\log \log N}{\log N}\right)},
	\end{equation*}so with probability at least $1-2/(\log N)$:
	\begin{flalign*}
	\max_{1 \leq i \leq N} Z_i & \leq \sqrt{2(v-c) \log N} + \sqrt{2\log \log N} \left(\sqrt{c}+ \sqrt{v-c}\right) \\
	& \leq  \sqrt{2(v-c) \log N} + 2 \sqrt{v \log \log N},
	\end{flalign*} where we used $\sqrt{c}+\sqrt{v-c} \leq \sqrt{2v}$ in the last step.
\end{proof}

\subsection{Proof of Lemma \ref{lemma:final_function_study}}\label{appendix:final_function}
\begin{proof}
	For $\alpha \in (0,1]$,
	\begin{flalign*}
	\eqref{eq:final_function_study} & \iff \alpha^2(2-\alpha)^2 \geq 2 \alpha \left(\alpha(2-\alpha)-f(\alpha)\right)\\
	& \iff f(\alpha) \geq \alpha^2-\alpha^3/2. 
	\end{flalign*} The inequality is verified for $\alpha < 1/2$. To conclude the proof of \eqref{eq:final_function_study}, it remains to check that for $1\geq \alpha \geq 1/2$, $f(\alpha) \geq \alpha^2-\alpha^3/2$, which is equivalent to 
	\begin{flalign*}
	\alpha^2 - \frac{1}{2}(2\alpha -1)^2 \geq \alpha^2-\alpha^3/2 & \iff \alpha^3 - 4\alpha^2 + 4\alpha -1 \geq 0\\
	& \iff (\alpha-1)(\alpha^2-3\alpha+1) \geq 0\\
	& \iff  \alpha^2-3\alpha+1 \leq 0 \iff \alpha \geq \frac{3 - \sqrt{5}}{2} \sim 0.382...
	\end{flalign*} 
\end{proof}

\subsection{Proof of Lemma \ref{lemma_ineq_PZ}: Payley-Zygmund inequality}\label{appendix:PZ_ineq}
\begin{proof}
	Using Cauchy-Schwarz inequality,
	\begin{flalign*}
	\dE[Y] & = \dE\left[Y\ind{Y < c \, {\dE\left[Y\right]}} \right] + \dE\left[Y\ind{Y \geq c \, {\dE\left[Y\right]}} \right]\\
	& \leq c \, {\dE\left[Y\right]} + \dE[Y^2]^{1/2} \, \dP(Y \geq c \, \dE\left[Y\right])^{1/2},
	\end{flalign*} which gives $\dE[Y]^2 \left(1 - c \right)^2 \leq \dE[Y^2] \, \dP(Y \geq c \, {\dE\left[Y\right]})$.
\end{proof} 

\subsection{Proof of Lemma \ref{lemma:control_dev_cor_gaussian}: Control of deviation probabilities for correlated Gaussians}\label{appendix:corr_gaussians}

\begin{proof} Let us first make a change of variable which preserves the joint distribution:
	\begin{flalign*}
	(Z_1,Z_2) = (Z , \alpha_n Z + \sqrt{1-\alpha_n^2} Z'),
	\end{flalign*} with $Z,Z'$ two independent standard Gaussian variables. \\
	
	\underline{\textbf{Proof of $(i)$:}} Note that standard Gaussian concentration gives $\dP\left(Z>2t_n \big| Z>t_n \right) \sim \frac{1}{2}e^{-3t_n^2/2}$. Thus, for $n$ large enough
	\begin{flalign*}
	\dP \left(Z_1 > t_n, Z_2 > t_n \right) & \leq \dP \left(Z > t_n \right) e^{-3t_n^2/2} + \dP \left(Z > t_n \right)\dP \left(\alpha_n Z + \sqrt{1-\alpha_n^2} Z' > t_n, Z\leq 2t_n \big| Z>t_n \right) \\
	& \leq e^{-2t_n^2} + \dP \left(Z > t_n \right)\dP \left(Z' > t_n -2\alpha_n t_n + O(t_n \alpha_n^2)\right)\\
	& \leq e^{-2t_n^2} + \dP \left(Z > t_n \right)\dP \left(Z' > t_n -o(1)\right)\\
	& \leq e^{-2t_n^2} +  (1+o(1))\dP \left(Z > t_n\right) \dP \left(Z' > t_n\right)\\
	& = e^{-2t_n^2} +  (1+o(1))\dP \left(Z_1 > t_n\right) \dP \left(Z_2 > t_n\right).
	\end{flalign*}
	
	\underline{\textbf{Proof of $(ii)$:}} For any $(s_n)$ such that $s_n \leq  t_n$ for all $n$, one has
	\begin{flalign*}
	\dE\left[e^{s_n Z} \, \big| \, Z > t_n \right] & = \frac{1}{\sqrt{2 \pi}} \int_{t_n}^{+ \infty} e^{s_n z-z^2/2} dz \left({\frac{1}{\sqrt{2 \pi}} \int_{t_n}^{+ \infty} e^{-z^2/2} dz} \right)^{-1}\\
	& = {e^{s_n^2/2} \int_{t_n - s_n}^{+ \infty} e^{-z^2/2} dz} \left({\int_{t_n}^{+ \infty} e^{-z^2/2} dz}\right)\\
	& \sim \frac{t_n}{t_n - s_n} \exp \left(s_n^2/2 - (t_n-s_n)^2/2 + t_n^2/2\right) = \frac{t_n}{t_n - s_n} e^{s_n t_n}.
	\end{flalign*} Using independence of $Z,Z'$ and Chernoff bound, we get, taking $s_n$ such that $\alpha s_n = u t_n$ with $u<1$, for $n$ large enough,
	\begin{flalign*}
	\dP\left(\alpha Z + \sqrt{1-\alpha^2} Z' > t_n	\big | Z >t_n \right) & 
	\leq (1+o(1))\frac{t_n}{t_n-\alpha s_n} \exp\left( \alpha s_n t_n + \frac{1-\alpha^2}{2} s_n^2 - s_n t_n \right)\\
	& \leq (1+o(1))\frac{1}{1-u} \exp\left(\left(u   + \frac{u^2 (1-\alpha^2)}{2\alpha^2} - \frac{u}{\alpha}\right) t_n^2 \right)\\
	& \overset{(a)}{\leq} (1+o(1)) (1+\alpha) \exp\left(- \frac{1-\alpha}{1+\alpha} \cdot \frac{t_n^2}{2} \right)
	\end{flalign*} where we took $u = \frac{\alpha}{1+\alpha} <1$ in $(a)$.  The proof follows from this last inequality, together with the bound $\dP\left( Z >t_n \right) \leq \frac{1}{\sqrt{2\pi}t_n} \exp\left(-\frac{t_n^2}{2}\right).$
\end{proof}

\end{document}